%% file: bandits_main.tex
\documentclass[11pt,reqno]{amsart}

\input{bandits_preamble} 
\begin{document}

\input{bandits_info} 
\input{bandits_abstract}
\maketitle
\input{bandits_introduction}
\input{bandits_prob_setup}
\input{bandits_main_idea}

\input{bandits_analysis}

\input{bandits_conclusion}
\bibliographystyle{unsrt}
\bibliography{bandits_references}
\input{bandits_appendix.tex}
\end{document}

%% file: bandits_preamble.tex
\usepackage[margin=1in]{geometry}
\usepackage{amsmath}%
\usepackage{amsfonts}%
\usepackage{amssymb}%
\usepackage{amsaddr}
\usepackage{graphicx}
\usepackage{subfig}
\usepackage{color}
\usepackage{enumitem}
\usepackage{algpseudocode}
\usepackage{algorithmicx}
\usepackage{algorithm}
\usepackage{bbm}

\newtheorem{theorem}{Theorem}
\theoremstyle{plain}

\newtheorem{corollary}{Corollary}

\newtheorem{lemma}{Lemma}

\newtheorem{proposition}{Proposition}
\newtheorem{remark}{Remark}

\numberwithin{equation}{section}

\makeatletter
\newcommand\footnoteref[1]{\protected@xdef\@thefnmark{\ref{#1}}\@footnotemark}
\makeatother

\newcommand{\hemant}[1]{\textcolor{black}{#1}}

\newcommand\given[1][]{\:#1\vert\:}

\newcommand{\calN}{\ensuremath{\mathcal{N}}}

\newcommand{\calP}{\ensuremath{\mathcal{P}}}

\newcommand{\calX}{\ensuremath{\mathcal{X}}}

\newcommand{\calE}{\ensuremath{\mathcal{E}}}

\newcommand{\ptil}{\ensuremath{\tilde{p}}}

\newcommand{\argmax}[1]{\underset{#1}{\operatorname{argmax}}}
\newcommand{\argmin}[1]{\underset{#1}{\operatorname{argmin}}}
\newcommand{\matR}{\ensuremath{\mathbb{R}}}
\newcommand{\expec}{\ensuremath{\mathbb{E}}}
\newcommand{\mathatX}{\ensuremath{\widehat{\mathbf{X}}}}
\newcommand{\mathatA}{\ensuremath{\widehat{\mathbf{A}}}}

\newcommand{\matA}{\ensuremath{\mathbf{A}}}
\newcommand{\matB}{\ensuremath{\mathbf{B}}}
\newcommand{\matG}{\ensuremath{\mathbf{G}}}
\newcommand{\matX}{\ensuremath{\mathbf{X}}}
\newcommand{\matI}{\ensuremath{\mathbf{I}}}
\newcommand{\matM}{\ensuremath{\mathbf{M}}}
\newcommand{\matU}{\ensuremath{\mathbf{U}}}
\newcommand{\matV}{\ensuremath{\mathbf{V}}}

\newcommand{\norm}[1]{\parallel{#1}\parallel}
\newcommand{\abs}[1]{|{#1}|}
\newcommand{\set}[1]{\left\{{#1}\right\}}
\newcommand{\dotprod}[2]{\langle#1,#2\rangle}
\newcommand{\grad}{\ensuremath{\bigtriangledown}}

\newcommand{\vecx}{\mathbf x}
\newcommand{\vecN}{\mathbf N}
\newcommand{\vecH}{\mathbf H}

\newcommand{\vecu}{\mathbf u}

\newcommand{\vecy}{\mathbf y}

\newcommand{\veca}{\mathbf a}

\newcommand{\veco}{\mathbf 0}

%% file: bandits_info.tex
\title[Stochastic Continuum armed bandit problem of few linear parameters in high
dimensions]{Stochastic Continuum armed bandit problem of few linear parameters in high
dimensions}
\author[H.Tyagi, S.Stich, B.G\"{a}rtner]{Hemant Tyagi, Sebastian Stich and Bernd G\"artner }

\address{Department of Computer Science, \\
Institute of Theoretical Computer Science, \\
ETH Z\"{u}rich, CH-8092, Switzerland}

\thanks{The project CG Learning acknowledges the financial support of the Future and Emerging Technologies (FET)
programme within the Seventh Framework Programme for Research of the European
Commission, under FET-Open grant number: 255827. \\
This is part of a \hemant{journal paper\cite{Tyagi14_tocs}} accepted in: Theory of Computing 
Systems (TOCS), 2014 - Special issue on WAOA 2013. 
\hemant{It is also part of HT's PhD thesis \cite{TyagiPhd}. See also Remark \ref{rem:disc_ear_ver_darft}.}} 

\email{htyagi@inf.ethz.ch, sstich@inf.ethz.ch, gaertner@inf.ethz.ch}

\keywords{Bandit problems, continuum armed bandits, functions of few variables,
online optimization, low-rank matrix recovery}

%% file: bandits_abstract.tex
\begin{abstract}
We consider a stochastic continuum armed bandit problem where the arms are indexed by the 
$\ell_2$ ball $B_{d}(1+\nu)$ of radius $1+\nu$ in $\mathbb{R}^d$. The reward functions $r :B_{d}(1+\nu) \rightarrow \mathbb{R}$
are considered to intrinsically depend on $k \ll d$ unknown linear parameters so that 
$r(\vecx) = g(\matA \vecx)$ where $\matA$ is a full rank $k \times d$ matrix. Assuming the mean reward function
to be smooth we make use of results from low-rank matrix
recovery literature and derive an efficient randomized algorithm which achieves a regret bound of $O(C(k,d) n^{\frac{1+k}{2+k}} (\log n)^{\frac{1}{2+k}})$. 
Here $C(k,d)$ is at most polynomial in $d$ and $k$ and $n$ is the number of rounds or the sampling budget
which is assumed to be known beforehand.
\end{abstract}  

%% file: bandits_introduction.tex
\section{Introduction} \label{sec:intro}
In the continuum armed bandit problem, a player is given a set of strategies $S$---typically
a compact subset of $\mathbb{R}^d$. At each round $t=1,\dots,n$, the player chooses a strategy 
$\vecx_t$ from $S$ and then receives a reward $r_t(\vecx_t)$. Here $r_t:S \rightarrow \mathbb{R}$
is the reward function chosen by the environment at time $t$ according to the underlying model.
The model we consider in this work is \textit{stochastic} i.e. the reward functions are assumed 
to be sampled in an i.i.d manner from an underlying distribution at each round. The player selects
strategies across different rounds with the goal of maximizing the total expected reward. 
Specifically, the performance of the player is measured in terms of
regret defined as the difference between the total expected reward of the best fixed (i.e. not varying
with time) strategy and the expected reward of the sequence of strategies played by the player. If the regret
after $n$ rounds is sub-linear in $n$, this implies as $n \rightarrow \infty$ that the per-round expected reward of the player
asymptotically approaches that of the best fixed strategy.

The problem faced by the player at each round is the classical ``exploration-exploitation dilemma''. On one hand if the
player chooses to focus his attention on a particular strategy which he considers to be the best (``exploitation'')
then he might fail to know about other strategies which have a higher expected reward. However
if the player spends too much time collecting information (``exploration'') then he might fail to
play the optimal strategy sufficiently often. Some applications of continuum armed bandit problems are in: 
(i) online auction mechanism design \cite{Blum03,Kleinberg03}
where the set of feasible prices is representable as an interval and, (ii)
online oblivious routing \cite{Bansal03} where $S$ is a flow polytope. 

For a $d$-dimensional strategy space, if the only assumption made on the reward
functions is on their degree of smoothness then any algorithm
will incur worst-case regret which depends exponentially on $d$ \cite{Kleinberg04}. To see this,
let $S=[-1,1]^d$ and consider a time invariant reward function that is zero
in all but one orthant $\mathcal{O}$ of $S$. More precisely, let $R(n)$ denote
the cumulative regret incurred by the algorithm after $n$ rounds. 
Bubeck et al. \cite{Bubeck2011ALT} showed that
$R(n) = \Omega(n^{\frac{d+1}{d+2}})$ after $n=\Omega(2^{d})$ plays for stochastic
continuum armed bandits\footnote{rewards sampled at each round in an i.i.d manner from
an unknown probability distribution.}
with $d$-variate Lipschitz continuous mean reward functions
defined over $[0,1]^d$. Clearly the per-round expected regret $R(n)/n = \Omega(n^{\frac{-1}{d+2}})$ 
which means that it converges to zero at a rate at least exponentially slow in $d$. 
This curse of dimensionality is avoided by reward functions possessing more structure, 
two popular cases being linear reward functions (see for example
\cite{McMahan04,Abernethy08}) and convex reward functions (see for example
\cite{Flaxman05,Kleinberg04}) for which the regret is \textit{polynomial} in $d$ and
sub-linear in $n$.

\subsection*{Low dimensional models for high dimensional reward functions} Recently there has been work
in the online optimization literature where the reward functions are assumed to be
low-dimensional or in other words have only a few degrees of freedom compared to the
ambient dimension. In \cite{Carpentier12,Yadkori12} the authors consider the linear stochastic bandit
problem in the setting that the unknown parameter (of dimension $d$) is $k$-sparse with $k \ll d$.
In \cite{Tyagi2013} the authors consider both stochastic and adversarial versions of continuum armed bandits
where the $d$-variate reward functions are assumed to depend on an unknown subset 
of the coordinate variables of size $k \ll d$. They derive nearly optimal regret bounds with the rate
of regret depending only on $k$. In \cite{Chen12} the authors consider the problem of Bayesian optimization of 
high dimensional functions by again assuming the functions to depend on only a few relevant variables.
Considering the function to be a sample from a high dimensional Gaussian process they provide
an algorithm with strong theoretical guarantees in terms of regret bounds. This model is generalized in
\cite{Wang13} where the authors consider the underlying function to effectively vary along a low-dimensional
subspace. Assuming the noise-less setting they adopt a Bayesian optimization framework and derive
bounds on \textit{simple regret}.

We consider the setting where the reward function $r_t:B_d(1+\nu)
\rightarrow \matR$ at each time $t$ depends on an unknown collection of $k \ll d$ linear parameters
implying $r_t(\vecx) = g_t(\matA\vecx)$ where $\matA \in \mathbb{R}^{k \times d}$ is full rank. 
This model can be seen as a generalization of \cite{Tyagi2013} where the reward functions
were modeled as $r(x_1,\dots,x_d) = g(x_{i_1},\dots,x_{i_k})$. Thus in the special case
where each row of $\matA$ has a single $1$ and $0$'s otherwise, we arrive at the setting of 
\cite{Tyagi2013}. There has also been significant effort in other fields to develop tractable
algorithms for \textit{approximating} $d$ variate functions (with $d$ large) from point
queries by assuming the functions to intrinsically depend on a few variables or
parameters (cf. \cite{Devore2011,Belkin03,Coifman06,Greenshtein06} and
references within). In particular the authors in \cite{Fornasier2012,Tyagi2012_nips}
considered the problem of approximating functions of the form $f(\vecx) = g(\matA\vecx)$
from point queries.

Very recently and independently a work parallel to ours \cite{Djolonga13} 
considered the same bandit problem as ours i.e. they also assume the $d$-variate reward 
functions to depend on $k \ll d$ unknown linear parameters. Although they consider the mean 
reward function to reside in a RKHS (Reproducible Kernel Hilbert space) and adopt a Bayesian
optimization framework, the scheme they employ is similar to ours. We comment on their results 
in the concluding remarks section towards the end.

\subsection*{Other related Work} The continuum armed bandit problem was first
introduced in \cite{Agrawal95} for the case $d=1$ where an algorithm achieving a
regret bound of $o(n^{(2\alpha+1)/(3\alpha+1)+\eta})$ for any $\eta > 0$ was
proposed for local H\"{o}lder continuous mean reward functions with exponent $\alpha \in (0,1]$. In
\cite{Kleinberg03} a lower bound of $\Omega(n^{1/2})$ was proven for this
problem. This was then improved upon in \cite{Kleinberg04} where the author
derived upper and lower bounds of $O(n^{\frac{\alpha+1}{2\alpha+1}}(\log
n)^{\frac{\alpha}{2\alpha+1}})$ and $\Omega(n^{\frac{\alpha+1}{2\alpha+1}})$
respectively. In \cite{Cope09} the author considered a class of mean reward
functions defined over a compact convex subset of $\matR^d$ which have (i) a
unique maximum $\vecx^{*}$, (ii) are three times continuously differentiable and
(iii) whose gradients are well behaved near $\vecx^{*}$. It was shown that a
modified version of the Kiefer-Wolfowitz algorithm achieves a regret bound of
$O(n^{1/2})$ which is also optimal. In \cite{Auer07improvedrates} the $d = 1$
case was treated, with the mean reward function assumed to only satisfy a local
H\"{o}lder condition around the maxima $\vecx^{*}$ with exponent $\alpha \in
(0,\infty)$. Under these assumptions the authors considered a modification of
Kleinberg's CAB1 algorithm \cite{Kleinberg04} and achieved a regret bound of
$O(n^{\frac{1+\alpha-\alpha\beta}{1+2\alpha-\alpha\beta}} (\log
n)^{\frac{\alpha}{1+2\alpha-\alpha\beta}})$ for some known $0 < \beta < 1$. In
\cite{Kleinberg08,Bubeck2011} the authors studied a very general setting for the
multi-armed bandit problem in which $S$ forms a metric space, with the reward
function assumed to satisfy a Lipschitz condition with respect to this metric.

\subsection*{Our Contributions} Our main contribution is to derive an algorithm namely 
CAB-LP(d,k) which achieves an upper bound of $O(C(k,d) n^{\frac{1+k}{2+k}} (\log n)^{\frac{1}{2+k}})$ on the regret
after $n$ rounds. The factor $C(k,d) = O(\text{poly}(k) \cdot \text{poly}(d))$, captures the uncertainty of
not knowing the $k$-dimensional sub-space spanned by the rows of $\matA$.
This bound is derived for a slightly restricted class of Lipschitz continuous mean reward functions.
In terms of $n$, it nearly matches
the $\Omega(n^{\frac{1+k}{2+k}})$ lower bound \cite{Bubeck2011ALT}, 
for $k$-variate Lipschitz continuous mean reward functions. As explained earlier, the per-round regret $R(n)/n$ 
approaches zero (as $n$ increases), at a rate exponential in $k$. Thus for $k \ll d$, we avoid the curse of dimensionality.
We assume $n$ to be known to the algorithm (hence it is not anytime) and refer to it as the sampling budget.
The main idea of the algorithm is to first use a fraction of the budget for estimating
the unknown $k$-dimensional sub-space spanned by the rows of the linear parameter matrix $\matA$. After obtaining
this estimate we then employ the CAB1 algorithm \cite{Kleinberg04} which is restricted to
play strategies only from the estimated subspace. To derive sub-linear regret bounds we show
that a careful allocation of the sampling budget is necessary between the two phases.

\subsection*{Organization of the paper} The rest of the paper is organized as follows. In Section \ref{sec:prob_setup}
we state the problem formally. Next we explain the main intuition behind our approach along with our main results in 
Section \ref{sec:main_idea_res}. In Section \ref{sec:analysis} we provide a formal analysis of our approach and derive regret
bounds. Finally we provide concluding remarks in Section \ref{sec:conclusion}.

%% file: bandits_prob_setup.tex
\section{Problem Setup} \label{sec:prob_setup}
We assume that a set of strategies $S$ is available to the player. For our purposes $S$ is considered 
to be the $\ell_2$-ball of radius $1 + \nu$ for some $\nu > 0$, denoted as $B_{d}(1+\nu)$. At each time $t = 1,\dots,n$ the environment 
chooses a reward function $r_t:B_d(1+\nu) \rightarrow \mathbb{R}$. Upon playing the strategy $\vecx_t$ the player receives
the reward $r_t(\vecx_t)$. Here the number of rounds $n$ (sampling budget) is assumed to be known to the player.
We consider the setting where each $r_t$ depends on $k \ll d$ unknown linear parameters $\veca_1,\dots,\veca_k \in \matR^d$ with
$k$ assumed to be known to the player. 
In particular, denoting $\matA = [\veca_1 \dots \veca_k]^T \in \mathbb{R}^{k \times d}$ we assume that $r_t(\vecx) = g_t(\matA\vecx)$.

\hemant{The reward functions $g_t$ are considered to be samples from some fixed but unknown probability distribution over functions
$g : B_k(1+\nu) \rightarrow \mathbb{R}$. We then have the expected reward function as $\bar{g}(\vecu) = \expec[g(\vecu)]$
where $\vecu \in B_{k}(1+\nu)$. We consider a \emph{specific} instance of this model where 
\begin{equation} \label{eq:stoch_lin_rewmod}
r_t(\vecx) = \bar{g}(\matA\vecx) + \eta_t \ ; \quad t=1,2,\dots,n
\end{equation}
and $(\eta_t)_{t=1}^{n}$ is i.i.d Gaussian noise with mean $\expec[\eta_t] = 0$ and variance $\expec[\eta_t^2] = \sigma^2$. 
Hence we associate with each arm $\vecx \in B_{d}(1+\nu)$, a normal distribution: $\calN(\bar{g}(\matA\vecx), \sigma^2)$ for the corresponding reward.}

We assume $\bar{g}$ to be sufficiently smooth - in particular to be two times continuously differentiable.
Specifically, we assume for some constant $C_2 > 0$ that the magnitude of all partial derivatives of $\bar{g}$, up to order two,
are bounded by $C_2$:
\begin{equation} \label{eq:prob_setup_smooth_assump}
\text{sup}_{\abs{\beta} \leq 2} \norm{D^{\beta} \bar{g}}_{\infty} \leq C_2 \ ; \quad 
D^{\beta}\bar{g} = \frac{\partial^{|\beta|}\bar{g}}{\partial y_1^{\beta_1}\dots \partial y_k^{\beta_k}} \ , 
\ \abs{\beta} = \beta_1 + \dots + \beta_k.
\end{equation}
Note that this is slightly stronger then assuming 
Lipschitz continuity\footnote{Indeed for a compact domain, any $C^2$ function is Lipschitz continuous but the converse is not 
necessarily true. Therefore, the mean reward functions that we consider, belong to a slightly restricted class of Lipschitz continuous 
functions.}.
We now make additional assumptions on the mean reward function $\bar{g}$. In fact it was shown by Fornasier et al. \cite{Fornasier2012}
that such additional assumptions are also necessary in order to formulate a tractable algorithm.
For example when $k = 1$, if we only make smoothness assumptions on $\bar{g}$, then one can construct $\bar{g}$ so that $\Omega(2^d)$ many samples
are needed to distinguish between $\bar{r}(\vecx) \equiv 0$ and $\bar{r}(\vecx) \equiv \bar{g}(\veca^T\vecx)$ \cite{Fornasier2012}.

To this end, we define the following matrix:
\begin{equation}
H^r := \int_{\mathbb{S}^{d-1}} \nabla\bar{r}(\vecx) \nabla\bar{r}(\vecx)^{T} d\vecx 
= \matA^T \cdot \int_{\mathbb{S}^{d-1}} \nabla\bar{g}(\matA\vecx) \nabla\bar{g}(\matA\vecx)^{T} d\vecx \cdot \matA
\end{equation}
where the second equality follows from the identity $\nabla \bar{r}(\vecx) = \matA^T\nabla\bar{g}(\matA\vecx)$. 
Let $\sigma_i(H^r)$ denote the $i^{th}$ singular value of $H^{r}$. We make a technical
assumption related to the conditioning of $H^r$. This assumption allows us to derive a tractable algorithm for our problem. 
We assume for some $\alpha > 0$ that:
\begin{equation} \label{eq:sing_vals_cond_mat}
\sigma_1(H^{r}) \geq \sigma_2(H^{r}) \geq \dots \geq \sigma_k(H^{r}) \geq \alpha > 0. 
\end{equation}
The parameter $\alpha$ determines the tractability of our algorithm. As explained in Section \ref{subsec:tract_param_alpha},
there are interesting function classes that satisfy \eqref{eq:sing_vals_cond_mat} for usable values of $\alpha$.

Following Fornasier et al. \cite{Fornasier2012}, we also assume without loss of generality,
$\matA$ to be row orthonormal so that $\matA \matA^T = \matI$. Indeed if this is not the case then through SVD (singular
value decomposition) of $\matA$ we obtain 
$\matA = \underbrace{\matU}_{k \times k} \underbrace{\Sigma}_{k \times k} \underbrace{\matV^T}_{k \times d}$
where $\matU,\Sigma, \matV^{T}$ are unitary, diagonal and row-orthonormal matrices respectively. Therefore we obtain
\begin{equation*}
\bar{r}(\vecx) = \bar{g}(\matA\vecx) = \bar{g}(\matU\Sigma\matV^T\vecx) = \bar{g}^{\prime}(\matV^T\vecx) 
\end{equation*}
where $\bar{g}^{\prime}(\vecy) = \bar{g}(\matU\Sigma\vecy)$ for $\vecy \in B_{k}(1+\nu)$. Hence within a scaling 
of the parameter $C_2$ by a factor depending polynomially on $k,\sigma_1(\matA)$ we can assume $\matA$ to be
row-orthonormal.

%
\noindent \textbf{Regret after $n$ rounds.} After $n$ rounds of play the cumulative expected regret is defined as:
\begin{equation}
R(n) = \sum_{i=1}^{n}\expec[r_t(\vecx^{*}) - r_t(\vecx_t)]  = \sum_{i=1}^{n}[\bar{g}(\matA\vecx^{*})-\expec[\bar{g}(\matA\vecx_t)]],
\end{equation}
where $\vecx^{*}$ is the optimal strategy belonging to the set
\begin{equation}
  \argmax{\vecx \in B_d(1+\nu)}\expec[r_t(\vecx)] = \argmax{\vecx \in B_d(1+\nu)} \bar{g}(\matA\vecx)
\end{equation}
Here $\vecx_1,\vecx_2,\dots,\vecx_n$ is the sequence of strategies played by the algorithm;  
\hemant{the expectation is defined over the randomness of the environment and the internal randomness of the algorithm.} 
The goal of the algorithm is to minimize regret i.e. ensure $R(n) = o(n)$ so that $\lim_{n \rightarrow \infty} R(n)/n = 0$.

%% file: bandits_main_idea.tex
\section{Main idea and Results} \label{sec:main_idea_res}
The main idea behind our algorithm is to proceed in two phases namely : (i) \textbf{PHASE 1} where we use a fraction of the sampling 
budget $n$ to recover an estimate of the ($k$ dimensional) subspace spanned by the rows of $\matA$ and then (ii) \textbf{PHASE 2} where we employ a standard 
continuum armed bandit algorithm that plays strategies from the previously estimated $k$ dimensional subspace. 

Intuitively we can imagine that the closer the estimated subspace is to the original one, the closer will the regret bound
achieved by the CAB algorithm be to the one it would have achieved by playing strategies from the unknown $k$-dimensional subspace.
However one should be careful
here since spending too many samples from the budget $n$ on \textbf{PHASE 1} can lead to regret which is $\Theta(n)$. 
On the other hand if the recovered subspace
is a bad estimate then it can again lead to $\Theta(n)$ regret since the optimization carried out in \textbf{PHASE 2} would be rendered meaningless. 

Hence it is important to carefully divide the sampling budget between the two phases in order to guarantee a regret bound that is sub-linear in $n$.
We now describe these two phases in more detail and outline the above idea formally.
%
\begin{enumerate}

\item \textbf{PHASE 1}(Subspace recovery phase.) In this phase we use the first $n_1 (< n)$ samples from our budget to generate an estimate 
$\widehat{\matA} \in \mathbb{R}^{k \times d}$ of $\matA$ such that the row space of $\widehat{\matA}$ is close to that of $\matA$. In particular we measure 
this closeness in terms of the Frobenius norm implying that we would like $\norm{\matA^T\matA - \widehat{\matA}^T\widehat{\matA}}_F$ to be sufficiently small.
Denoting the total regret in this phase by $R_1$ we then have that:
\begin{equation}
R_1 = \sum_{t=1}^{n_1}[\bar{r}(\vecx^{*}) - \expec[\bar{r}(\vecx_t)]] = O(n_1).   
\end{equation}		
This follows trivially since $\bar{r}$ is a smooth function defined over a compact domain.
We can see that $n_1$ should necessarily be $o(n)$ otherwise the total regret would be dominated by $R_1$ leading to linear regret. 

%
\item \textbf{PHASE 2}(Optimization phase.) Say that we have in hand an estimate $\widehat{\matA}$ from \textbf{PHASE 1}. We now employ a standard
CAB algorithm that is restricted to play strategies from the row space of $\widehat{\matA}$. Let us denote $n_2 = n - n_1$ to be the duration
of this phase and $\calP \subset B_d(1+\nu)$ where
\begin{equation*}
\calP := \set{\widehat{\matA}^T\vecy \in \mathbb{R}^d: \vecy \in B_{k}(1+\nu)}.
\end{equation*}
The CAB algorithm will play strategies only from $\calP$ and therefore will strive to optimize against the optimal strategy 
$\vecx^{**} = \widehat{\matA}^T\vecy^{**} \in \calP$ where
\begin{equation*}
\vecy^{**} \in \argmax{\vecy \in B_{k}(1+\nu)} \bar{g}(\matA\widehat{\matA}^T\vecy).
\end{equation*}
Furthermore we also observe that the total regret incurred in this phase can be written as:
\hemant{\begin{align} \label{eq:tot_reg_expr_optphase}
\sum_{t=n_1+1}^{n}[\bar{r}(\vecx^{*}) - \expec[\bar{r}(\vecx_t)]] 
&= \underbrace{\expec_{\mathatA}[\sum_{t=n_1+1}^{n}[\bar{r}(\vecx^{*}) - \bar{r}(\vecx^{**})]]}_{= R_3} \\ 
&+ \underbrace{\expec_{\mathatA}[\sum_{t=n_1+1}^{n}\expec[\bar{r}(\vecx^{**}) - \bar{r}(\vecx_t) \given \mathatA]]}_{= R_2}.
\end{align}}
%
Note that $R_2$ represents the expected regret incurred
by the CAB algorithm against the optimal strategy from $\calP$. In particular, we will obtain $R_2 = o(n-n_1)$. 

Next, the term $R_3$ captures the offset between the actual optimal strategy $\vecx^{*} \in B_d(1+\nu)$ and $\vecx^{**} \in \calP$.
In particular $R_3$ can be bounded by making use of: (i) the Lipschitz continuity of the mean reward $\bar{g}$ and, (ii)
the bound on the subspace estimation error : $\norm{\matA^T\matA - \widehat{\matA}^T\widehat{\matA}}_F$. This is shown
precisely in the form of the following Lemma, the proof of which is presented in the appendix.
\begin{lemma} \label{lem:init_bound_R3}
\hemant{For some $0 < f < 1$, denote the event $\calE = \set{\norm{\matA^T\matA - \widehat{\matA}^T\widehat{\matA}}_F \leq f}$. 
We have that $R_3 \leq O(n_2 \sqrt{k} f) + \mathbb{P}(\calE^{c}) O(k^{3/2} n_2)$ where $n_2 = n - n_1$ and 
$\calE^{c}$ is complement of $\calE$.}
\end{lemma}
\end{enumerate}

%
\hemant{
\begin{remark} \label{rem:disc_ear_ver_darft}
In the versions of this draft that were published in \cite{Tyagi14_tocs,TyagiPhd}, the term 
$\sum_{i=1}^{n} \expec[[\bar{g}(\matA\vecx^{*})-\bar{g}(\matA\vecx_t)]|\mathbbm{1_{\calE}}]$ ($\mathbbm{1_{\calE}}$ 
is the indicator variable w.r.t event $\calE$ defined in Lemma \ref{lem:init_bound_R3}) was considered as the regret, and was bounded w.h.p. 
Since it might be considered a bit unnatural to define regret in terms of such a conditional expectation, 
we translate the high probability bound into one in expectation. Thus $R_3$ in Lemma \ref{lem:init_bound_R3}, 
is now bounded in expectation. This results in a minor change in the statement of 
Theorem's \ref{thm:main_res_k_lin_params},\ref{thm:main_thm_reg_bds} compared to \cite{Tyagi14_tocs,TyagiPhd}; 
but the regret rate in terms of $n$ remains the same. 
\end{remark}
}
%
\subsection*{Main results} Our main result is to derive a randomized algorithm namely CAB-LP($d,k$) which achieves a regret bound of
$O(C(k,d) n^{\frac{1+k}{2+k}} (\log n)^{\frac{1}{2+k}})$ after $n$ rounds.
Here, $C(k,d) = O(\text{poly}(k) \cdot \text{poly}(d))$ accounts for the uncertainty of
not knowing the $k$-dimensional sub-space spanned by the rows of $\matA$.
We state this formally in the form of the following theorem below\footnote{This theorem is stated again in Section \ref{sec:analysis} for completeness.}. 
\begin{theorem} \label{thm:main_res_k_lin_params}
\hemant{Let the number of rounds $n$ satisfy $n = \Omega(\text{poly}(k) \cdot \text{poly}(d))$. For $k \geq 3$, 
assume that the parameter $\alpha$ depends polynomially on $d^{-1}$. 
Then algorithm CAB-LP($d,k$) achieves a total regret of
\begin{equation} \label{eq:main_res_bandit_linparam}
O\left(\frac{k^{13} d^2 \sigma^2 (\log n)^{4}}{\alpha^6} n^{\frac{4}{k+2}}
+ n^{\frac{1+k}{2+k}} (\log n)^{\frac{1}{2+k}} \right) 
\end{equation}  
after $n$ rounds.}
\end{theorem}
Recall that $\sigma$ denotes the variance of the external Gaussian noise $\eta$ in \eqref{eq:stoch_lin_rewmod} while $\alpha$ was defined
in \eqref{eq:sing_vals_cond_mat}. 
\hemant{The regret incurred in the first phase is the first term in \eqref{eq:main_res_bandit_linparam}. 
The regret incurred in the second phase corresponds to the second term in \eqref{eq:main_res_bandit_linparam}.
Note that the dependence of the regret bound in terms of $n$ is $O(n^{\frac{1+k}{2+k}} (\log n)^{\frac{1}{2+k}})$ when 
$k > 3$, which is close to the optimal rate. Indeed,} say the linear parameter matrix $\matA$, 
or even the sub-space spanned by its rows, was known. We then know a lower bound of 
$\Omega(n^{\frac{1+k}{2+k}})$ on regret, for $k$-variate Lipschitz continuous mean rewards \cite{Bubeck2011ALT}. 
In terms of $n$, our bound nearly matches this lower bound, albeit for a slightly restricted class of Lipschitz
continuous mean reward functions.
As discussed in Section \ref{sec:conclusion} it seems to be possible to
remove the \hemant{$(\log n)^{\frac{1}{2+k}}$ factor appearing in the bound} by using recent results for finite-armed bandits. 
Lastly we also note the dependence of our regret bound on the parameter $\alpha$.
As explained in Section \ref{subsec:tract_param_alpha}, 
$\alpha$ typically decreases as $d \rightarrow \infty$. Hence in order to obtain 
regret bounds that are at most polynomial in $d$ we would like $\alpha$ to be polynomial in $d^{-1}$. 
To this end, Proposition \ref{prop:tract_cond_alpha} in Section \ref{subsec:tract_param_alpha} which was proven by 
Tyagi et al. \cite{Tyagi2014_acha}, describes a fairly general class of functions for which $\alpha$ is 
$\Theta(d^{-1})$.

%% file: bandits_analysis.tex
\section{Analysis} \label{sec:analysis}
We now provide a thorough analysis of the two phase scheme discussed in the previous section. We start by first describing
a low-rank matrix recovery scheme which is used for obtaining an estimate of the unknown subspace represented by the row-space 
of $\matA$. 

\subsection{Analysis of sub-space recovery phase} \label{subsec:analysis_subsp_recov}
We first observe that the Taylor expansion of $\bar{r}$ around any $\vecx \in B_d(1+\nu)$ 
along the direction $\phi \in \mathbb{R}^d$ give us:
\begin{equation}
 \bar{r}(\vecx + \epsilon \phi) - \bar{r}(\vecx) = \epsilon\dotprod{\phi}{\grad\bar{r}(\vecx)} + 
\frac{1}{2}\epsilon^2\phi^T\grad^2\bar{r}(\xi)\phi \label{eq:taylors_exp_1}
\end{equation}
for any $\epsilon > 0$ and $\xi = \vecx + \theta\epsilon\phi$ with $0 < \theta < 1$. In particular by using $\grad \bar{r}(\vecx)
= \matA^T \grad\bar{g}(\matA\vecx)$ in \eqref{eq:taylors_exp_1} we obtain:
\begin{equation}
 \dotprod{\phi}{\matA^T\grad\bar{g}(\matA\vecx)} = \frac{\bar{r}(\vecx + \epsilon \phi) - \bar{r}(\vecx)}{\epsilon} 
- \frac{1}{2}\epsilon\phi^T\grad^2\bar{r}(\xi)\phi. \label{eq:taylors_exp_2}
\end{equation}
We now introduce the sampling scheme \footnote{The above sampling scheme was considered first in \cite{Fornasier2012} and later in \cite{Tyagi2012_nips}
for the problem of approximating functions of the form $f(\vecx) = g(\matA\vecx)$ from point queries.} by stating the choice of $\vecx$ and sampling
direction $\phi$ in \eqref{eq:taylors_exp_2}. We first construct
\begin{equation}
\calX := \set{\vecx_j \in \mathbb{S}^{d-1} \ ; \ j=1,\dots,m_{\calX}}. 
\end{equation}
This is the set of samples at which we consider the Taylor expansion of $\bar{r}$ as in \eqref{eq:taylors_exp_1}. In particular, we form
$\calX$ by sampling points uniformly at random from $\mathbb{S}^{d-1}$. Next, we construct 
the set of sampling directions $\Phi$ for $i=1,\dots,m_{\Phi}$, $j=1,\dots,m_{\calX}$ and $l = 1,\dots,d$ where:
\begin{equation}
 \Phi := \set{\phi_{i,j} \in B_d(\sqrt{d/m_{\Phi}}) : [\phi_{i,j}]_{l} = \pm\frac{1}{\sqrt{m_{\Phi}}} \ \text{with probability} \ 1/2}.
\end{equation}
Note that we consider $m_{\Phi}$ random sampling directions \textit{for each} point in $\calX$.
Hence we have that the total number of samples collected so far is
\begin{equation*}
 \abs{\calX} + \abs{\Phi} = m_{\calX} + m_{\calX}m_{\Phi} = m_{\calX}(m_{\Phi}+1).
\end{equation*}
Now note that at each time $1 \leq t \leq m_{\calX}(m_{\Phi}+1)$ upon choosing the strategy $\vecx_t$ we obtain the reward $r_t(\vecx_t) = \bar{r}(\vecx_t) + \eta_t$
where $\eta_t$ is i.i.d Gaussian noise. Therefore by first sampling at points $\vecx_1,\dots,\vecx_{m_{\calX}} \in \calX$ and
then sampling at $\vecx_j+\epsilon\phi_{1,j},\dots,\vecx_j+\epsilon\phi_{m_{\Phi},j}$ for each $\vecx_j$ we have from \eqref{eq:taylors_exp_2} 
the following for $i=1,\dots,m_{\Phi}$ and $j=1,\dots,m_{\calX}$.
\begin{equation}
\dotprod{\phi_{i,j}}{\matA^T\grad\bar{g}(\matA\vecx_j)} = \frac{r_{m_{\calX}+ij}(\vecx_j + \epsilon \phi_{i,j}) - r_j(\vecx_j)}{\epsilon} 
+\frac{\eta_j-\eta_{i,j}}{\epsilon} - \frac{1}{2}\epsilon\phi_{i,j}^T\grad^2 \bar{r}(\xi_{i,j})\phi_{i,j}. \label{eq:taylors_exp_3}
\end{equation}
We sum up \eqref{eq:taylors_exp_3} over all $j$ for each $i=1,\dots,m_{\Phi}$. This yields $m_{\Phi}$ equations
that can be summarized in the following succinct form:
\begin{equation}
\Phi(\matX) = \vecy + \vecN + \vecH. \label{eq:low_rank_meas_form}
\end{equation}
Here $\matX = \matA^T\matG$ where $\matG := [\grad\bar{g}(\matA\vecx_1)|\grad\bar{g}(\matA\vecx_2)|\cdots| \grad\bar{g}(\matA\vecx_{m_{\calX}})]_{k \times m_{\calX}}$.
Note that $\matX \in \mathbb{R}^{d \times m_{\calX}}$ has rank at most $k$. 
Next, $\Phi(\matX) := [\dotprod{\Phi_1}{\matX},\dots,\dotprod{\Phi_{m_{\Phi}}}{\matX}] \in \mathbb{R}^{m_{\Phi}}$ where
\begin{equation}
 \Phi_i =[\phi_{i,1} \phi_{i,2} \dots \phi_{i,m_{\calX}}] \in \mathbb{R}^{d \times m_{\calX}} 
\end{equation}
represents the $i^{\text{th}}$ measurement matrix and $\dotprod{\Phi_i}{\matX} = \text{Tr}(\Phi_i^T \matX)$ represents the $i^{\text{th}}$
measurement of $\matX$. The measurement vector is represented by $\vecy = [y_1 \dots y_{m_{\Phi}}] \in \mathbb{R}^{m_{\Phi}}$ where
\begin{equation} 
 y_i = \frac{1}{\epsilon} \sum_{j=1}^{m_{\calX}} \left(r_{m_{\calX}+ij}(\vecx_j + \epsilon \phi_{i,j}) - r_j(\vecx_j)\right). \label{eq:meas_vec_form}
\end{equation}
Lastly $\vecN = [N_1 \dots N_{m_{\Phi}}]$ and $\vecH = [H_1 \dots H_{m_{\Phi}}]$ represent the noise terms with
\begin{align*}
 N_i &= \frac{1}{\epsilon}\sum_{j=1}^{m_{\calX}} (\eta_j-\eta_{i,j}) \quad \text{(Stochastic noise)}, \\
 H_i &= -\frac{\epsilon}{2}\sum_{j=1}^{m_{\calX}} \phi_{i,j}^T \grad^2 \bar{r} (\xi_{i,j}) \phi_{i,j} \quad 
\text{(Noise due to non-linearity of} \ \bar{r}).
\end{align*}
Importantly, we observe that \eqref{eq:low_rank_meas_form} represents (noisy) linear measurements of the matrix $\matX$ which has rank $k \ll d$.
Hence by employing a standard solver for recovering low-rank matrices from noisy linear measurements we can hope to recover an approximation
$\widehat{\matX}$ to the unknown matrix $\matX$. Furthermore we note that information about the linear parameter matrix $\matA$ is encoded in
$\matX$. This intuitively suggests that one can hope to recover an approximation to $\matA$ with the help of $\widehat{\matX}$. 
In particular the closer $\mathatX$ is to $\matX$ the better will be the approximation to the row space of $\matA$. We now proceed
to demonstrate this formally.
%
\subsection*{Low-rank matrix recovery} As discussed, \eqref{eq:low_rank_meas_form} represents noisy measurements of the low rank matrix $\matX$ with
the linear operator $\Phi$. An important property of $\Phi$ is that it satisfies the so called Restricted Isometry Property (RIP) 
for low-rank matrices. This means that for all matrices $\matX_k$ of rank at most $k$:
\begin{equation}
(1-\delta_k)\norm{\matX_k}_F^2 \leq \norm{\Phi(\matX_k)}_2^2 \leq (1+\delta_{k})\norm{\matX_k}_F^2 \label{eq:rip_prop}
\end{equation}
holds true for some isometry constant $\delta_k \in (0,1)$. In general, any $\Phi$ that satisfies \eqref{eq:rip_prop} is said to have $\delta_k$-RIP.
In our case since $\Phi$ is a Bernoulli random measurement operator, it can be verified via standard covering arguments and concentration 
inequalities \cite{RechtFazel2010,Laurent2000} that $\Phi$  satisfies $\delta$-RIP for $0 < \delta_k < \delta < 1$ with probability at least
$1 - 2\exp(-m_{\Phi}q(\delta) + k(d+m_{\calX} + 1)u(\delta))$ where
\begin{equation*}
 q(\delta) = \frac{1}{144}\left(\delta^2 - \frac{\delta^3}{9}\right), \quad u(\delta) = \log\left(\frac{36\sqrt{2}}{\delta}\right).
\end{equation*}
An estimate of the low-rank matrix $\matX$ from the measurement vector $\vecy$ can be obtained through convex programming. For our
purposes we consider the following nuclear norm minimization problem also known as the matrix Dantzig selector (DS) \cite{Candes2010}.
%
\begin{equation}
\mathatX_{DS} = \argmin{} \norm{\matM}_{*} \ \text{s.t.} \ \norm{\Phi^{*}(\vecy - \Phi(\matM))} \leq \lambda. \label{eq:mat_DS_form}
\end{equation}
Here $\Phi^{*}: \mathbb{R}^{m_{\Phi}} \rightarrow \mathbb{R}^{d \times m_{\calX}}$ denotes the adjoint of the linear operator 
$\Phi : \mathbb{R}^{d \times m_{\calX}} \rightarrow \mathbb{R}^{m_{\Phi}}$. Furthermore for any matrix, $\norm{\cdot}_{*}$ and $\norm{\cdot}$
denote its nuclear norm (sum of singular values) and operator norm (largest singular value) respectively. By making use of the error bound for matrix DS 
presented as Theorem $1$ in \cite{Candes2010} we obtain the following result
on the performance of the matrix DS tuned to our problem setting. The proof is deferred to the appendix.
\begin{lemma} \label{lem:recov_res_DS}
Let $\mathatX_{DS} \in \mathbb{R}^{d \times m_{\calX}}$ denote the solution of \eqref{eq:mat_DS_form} and let $\mathatX_{DS}^{(k)}$ be the best rank $k$
approximation to $\mathatX_{DS}$ in the sense of $\norm{\cdot}_F$. Then for some constant $\gamma > 2\sqrt{\log 12}$, $0 < \delta_{4k} < \delta < \sqrt{2}-1$
we have that
\begin{equation*}
\norm{\mathatX_{DS}^{(k)} - \matX}_F \leq (C_0 k)^{1/2} \left(\frac{C_2 \epsilon d m_{\calX} k^2}{\sqrt{m_{\Phi}}} + 
\frac{8 \gamma \sigma \sqrt{m_{\calX} m_{\Phi} m}}{\epsilon} \right)(1 + \delta)^{1/2}
\end{equation*}
with probability at least $1 - 2\exp(-m_{\Phi}q(\delta) + 4k(d+m_{\calX}+1)u(\delta)) - 4\exp(-cm)$. Here $m = \max\set{d,m_{\calX}}$.
Furthermore the constants $C_0, c > 0$ depend on $\delta$ and $\gamma$ respectively.
\end{lemma}
\subsection*{Approximating row-space($\matA$)} Let's say we have\footnote{Ofcourse in practice we will not be able to solve
\eqref{eq:mat_DS_form} exactly,
but will instead obtain a solution that can be made to come arbitrarily close to the actual solution. This difference
will hence appear as an additional error term in the error bound of Lemma \ref{lem:recov_res_DS}.} in hand 
$\mathatX_{DS}^{(k)} \in \mathbb{R}^{d \times m_{\calX}}$ 
as the best rank $k$ approximation of the solution to \eqref{eq:mat_DS_form}. We can now obtain an estimate $\widehat{\matA}$ of 
row-space($\matA$) by setting $\widehat{\matA}^T$ to be equal to
the ($d \times k$) left singular vector matrix of $\mathatX_{DS}^{(k)}$. The quality of this estimation as measured by  
$\norm{\widehat{\matA}^T\widehat{\matA} - \matA^T\matA}_F$ was quantified in Lemma 2 of \cite{Tyagi2014_acha} for the noiseless case ($\sigma = 0$).
We adapt this result to our setting ($\sigma > 0$) and state it below. The proof is presented in the appendix.
\begin{lemma} \label{lem:recov_res_subspace}
For a fixed $0 < \rho < 1$, $m_{\calX} \geq 1$, $m_{\Phi} < m_{\calX} d$ let 
\begin{equation*}
a_1 = C_2 d k^2 , \quad b_1 = \frac{\sqrt{(1-\rho)\alpha}}{C_0^{1/2} (1+\delta)^{1/2} (\sqrt{k} + \sqrt{2})}. 
\end{equation*}
For any $0 < f < 1$ we then have for the choice
\begin{equation}
 \epsilon \in \left(\frac{f b_1-\sqrt{f^2 b_1^2 - 32 \gamma \sigma a_1\sqrt{m_{\calX} m}}}{2 a_1 \sqrt{m_{\calX}/m_{\Phi}}} , 
\frac{f b_1 + \sqrt{f^2 b_1^2 - 32 \gamma \sigma a_1 \sqrt{m_{\calX} m}}}{2 a_1 \sqrt{m_{\calX}/m_{\Phi}}} \right) \label{eq:eps_cond}
\end{equation}
that $\norm{\widehat{\matA}^T\widehat{\matA} - \matA^T\matA}_F \leq \frac{2f}{1 - f}$ holds true 
with probability at least
\begin{equation*}
1 - 2\exp(-m_{\Phi}q(\delta) + 4k(d+m_{\calX}+1)u(\delta)) - 4\exp(-cm) -k\exp\left(-\frac{m_{\calX}\alpha\rho^2}{2k C_2^2}\right).
\end{equation*}
\end{lemma}
We see in the above lemma that the step size parameter $\epsilon$ cannot be chosen to be arbitrarily small\footnote{In the absence of external stochastic noise 
(i.e. $\sigma = 0$) we can actually take $\epsilon$ to be arbitrarily small as shown in Lemma 2 of \cite{Tyagi2014_acha}}. In particular for $\epsilon$ too small
the stochastic noise will become prominent while for large $\epsilon$, the noise due to higher order Taylor's terms of the mean reward function
will start to dominate. 

\subsection*{Handling stochastic noise.} A point of obvious concern in Lemma \ref{lem:recov_res_subspace} is the condition required on the step size
parameter $\epsilon$ in \eqref{eq:eps_cond}. 
This condition is well defined if $f^2b_1^2 - 32 \gamma \sigma a_1 \sqrt{m_{\calX} m} > 0$. This would not have been a problem in the noiseless case
where $\sigma = 0$. A natural way to guarantee the well-posedness of \eqref{eq:eps_cond} is by \textit{re-sampling} and averaging the rewards
at each of the sampling points. Indeed if we consider each sampling point to be re-sampled $N$ times and then 
average the corresponding reward values, the variance of the stochastic noise will be reduced by a factor of $N$. 
By choosing a sufficiently large value of $N$, we can clearly ensure that
$f^2b_1^2 - 32 \gamma \sigma a_1 \sqrt{m_{\calX} m} > 0$ holds true. This is made precise in the following proposition which also states
a bound on the total regret $R_1$ suffered in this phase.
\begin{proposition} \label{prop:regret_bd_phase1}
Say that we resample $N$ times at each sampling point $\vecx_j \in \calX$ and $\vecx_j + \epsilon\phi_{i,j}$; 
$i=1,\dots,m_{\Phi}$ and $j=1,\dots,m_{\calX}$. Let the reward value at each sampling point be estimated as the 
average of the $N$ values. If $N > \frac{C^{\prime}k^6  d^2 \sigma^2 m_{\calX} m}{f^4 \alpha^2}$ for some
constant $C^{\prime} > 0$ (depending on $\rho,C_0,\delta,C_2,\gamma$) and with $m = \max\set{d,m_{\calX}}$, then \eqref{eq:eps_cond} in
Lemma \ref{lem:recov_res_subspace} is well defined. Consequently the total regret in \textbf{PHASE 1} is 
\begin{equation*}
 R_1 = O(n_1) = O(N m_{\calX} (m_{\Phi} + 1)) = O\left(\frac{k^6 d^2 \sigma^2}{\alpha^2} \frac{m_{\calX}^2 m_{\Phi} m}{f^4}\right) .
\end{equation*}
\end{proposition}
\begin{proof}
First note that \eqref{eq:eps_cond} in Lemma \ref{lem:recov_res_subspace} is well defined when
\begin{align*}
 f^2 b_1^2 - 32 \gamma \sigma a_1 \sqrt{m_{\calX} m} > 0 
\Leftrightarrow \sigma < \frac{f^2 b_1^2}{32 \gamma \sqrt{m_{\calX} m} \underbrace{C_2 d k^2}_{a_1}}.
\end{align*}
%
%
After plugging in the value of $b_1$ from Lemma \ref{lem:recov_res_subspace} we then obtain
\begin{equation}
 \sigma < \frac{f^2 b_1^2}{32 \gamma \sqrt{m_{\calX}m} C_2 d k^2} = \frac{C\alpha f^2}{(\sqrt{k}+\sqrt{2})^2 \sqrt{m_{\calX}m} d k^2} \label{eq:sigma_cond_1}
\end{equation}
where $C = \frac{(1-\rho)}{32 \gamma C_0 (1+\delta) C_2}$ is a constant.
Upon re-sampling $N$ times and subsequent averaging of reward values we have that the variance $\sigma$ changes to $\sigma/{\sqrt{N}}$. Replacing $\sigma$
with $\sigma/\sqrt{N}$ in \eqref{eq:sigma_cond_1} we obtain the stated condition on $N$. Lastly, we note that as a consequence of re-sampling the duration
of \textbf{PHASE 1} i.e. $n_1$ is $N m_{\calX}(m_{\Phi} + 1)$ implying the stated bound on $R_1$.
\end{proof}
\subsection{Analysis of optimization phase} We now analyze \textbf{PHASE 2} i.e. the optimization phase of our scheme. This phase
runs during time steps $t = n_1 + 1,n_1+2,\dots,n$ where $n_1 = N m_{\calX} (m_{\Phi} + 1)$. Given an estimate $\mathatA$ of the 
row space of $\matA$ we now consider optimizing \textit{only} over points lying in the row space of $\mathatA$. 
In particular consider $\calP \subset B_d(1+\nu)$ where
\begin{equation*}
\calP := \set{\widehat{\matA}^T\vecy \in \mathbb{R}^d : \vecy \in B_{k}(1+\nu)}.
\end{equation*}
We employ a standard CAB algorithm that plays points only from $\calP$ and therefore strives to optimize against the optimal strategy 
$\vecx^{**} = \widehat{\matA}^T\vecy^{**} \in \calP$ where
\begin{equation*}
\vecy^{**} \in \argmax{\vecy \in B_{k}(1+\nu)} \bar{g}(\matA\widehat{\matA}^T\vecy).
\end{equation*}
Recall from Section \ref{sec:main_idea_res} that the total regret incurred in this phase can be written as:
 \hemant{\begin{align*}
\sum_{t=n_1+1}^{n}[\bar{r}(\vecx^{*}) - \expec[\bar{r}(\vecx_t)]] 
= \underbrace{\expec_{\mathatA}[\sum_{t=n_1+1}^{n} [\bar{r}(\vecx^{*}) - \bar{r}(\vecx^{**})]]}_{= R_3} 
+ \underbrace{\expec_{\mathatA}[\sum_{t=n_1+1}^{n}\expec[\bar{r}(\vecx^{**}) - \bar{r}(\vecx_t) \given \mathatA]]}_{= R_2}.
\end{align*}}
where $R_2$ is the regret incurred by the CAB algorithm and $R_3$ is the regret incurred on account of not playing strategies
from the row space of $\matA$.

\subsection*{Bounding $R_2$} In order to bound $R_2$ we employ the CAB1 algorithm \cite{Kleinberg04}, with the UCB-1 algorithm \cite{Auer02}
as the finite armed bandit algorithm. Recall that this phase runs for a duration of $n_2 = n - n_1$ time steps.
A straightforward generalization of the result by Kleinberg \cite[Theorem 3.1]{Kleinberg04} to $k$ dimensions then yields
\begin{equation}
R_2 = O(n_2 ^{\frac{1+k}{2+k}} (\log n_2)^{\frac{1}{2+k}}) = O(n^{\frac{1+k}{2+k}} (\log n)^{\frac{1}{2+k}}). \label{eq:bound_regret_r2}
\end{equation}
Indeed for any integer $M > 0$, we simply discretize $[-1-\nu,1+\nu]^k$ into $(2M+1)^k$ points, with step size $1/M$
in each direction. We retain only those points that lie in $B_k(1+\nu)$ and multiply each of these with $\widehat{\matA}^T$.
This gives us a finite subset of $\calP$ on which we employ the UCB-1 algorithm. Since the time duration $n_2$ is known,
therefore in a manner similar to the proof of \cite[Theorem 3.1]{Kleinberg04}, one can find an optimal value of $M$, 
for which the regret bound of \eqref{eq:bound_regret_r2} is attained. 
\subsection*{Bounding $R_3$} The term $R_3$ can be bounded from above by a straightforward combination of Lemma \ref{lem:init_bound_R3} with 
Lemma \ref{lem:recov_res_subspace}. Hence we state this in the form of the following proposition without proof.
\begin{proposition} \label{prop:bound_regret_r3}
For fixed $0 < \rho < 1$, $m_{\calX} \geq 1$, $m_{\Phi} < m_{\calX} d$ and $0 < f < 1$, let $\epsilon$ be
chosen to satisfy \eqref{eq:eps_cond}. This then implies that 
\hemant{$R_3 \leq O(n_2 \sqrt{k} f + n_2 p k^{3/2})$ where
\begin{equation*}
p = 2\exp(-m_{\Phi}q(\delta) + 4k(d+m_{\calX}+1)u(\delta)) + 4\exp(-cm) + k\exp\left(-\frac{m_{\calX}\alpha\rho^2}{2k C_2^2}\right).
\end{equation*}}
\end{proposition}

\subsection{Bounding the total regret} \label{subsec:cab_liparam_tot_reg_bd}
Finally, we have all the results sufficient to bound the total regret. Indeed by using bounds on $R_1,R_2,R_3$
from Proposition \ref{prop:regret_bd_phase1}, \eqref{eq:bound_regret_r2} and Proposition \ref{prop:bound_regret_r3} respectively we have that:
\begin{equation}
R_1+R_2+R_3 = O\left(\frac{k^6 d^2 \sigma^2}{\alpha^2} \frac{m_{\calX}^2 m_{\Phi} m}{f^4} + 
n^{\frac{1+k}{2+k}} (\log n)^{\frac{1}{2+k}} + n_2 \sqrt{k} f \hemant{+ n_2 p k^{3/2}}\right). \label{eq:tot_regret_exp1}
\end{equation}
\hemant{where
\begin{equation} \label{eq:success_prob_exp}
p = 2\exp(-m_{\Phi}q(\delta) + 4k(d+m_{\calX}+1)u(\delta)) + 4\exp(-cm) + k\exp\left(-\frac{m_{\calX}\alpha\rho^2}{2k C_2^2}\right).
\end{equation}}
In order to bound the overall regret we need to choose the values of: $m_{\calX}, m_{\Phi}$ and $f$ carefully. We state these choices precisely
in the following theorem which is also our main theorem that provides a bound on the overall regret achieved by our scheme.
%
\begin{theorem} \label{thm:main_thm_reg_bds}
Under the assumptions and notations used thus far let: 

\begin{equation*}
f = \frac{1}{\sqrt{k}}\left(\frac{\log n}{n}\right)^{\frac{1}{k+2}}, 
\hemant{m_{\calX} = \frac{2k C_2^2}{\alpha\rho^2}\log(k/\ptil)} \quad \text{and} \quad  
m_{\Phi} = \frac{4k(d+m_{\calX}+1)u(\delta)c_1}{q(\delta)}
\end{equation*}

for \hemant{$\ptil = \frac{1}{n k^{3/2}}$ and some} $c_1 > 1$. 
\hemant{Assuming $k \geq 3$, let $n$ be sufficiently large, i.e., 
$n = \Omega(\text{poly}(k) \cdot \text{poly}(d))$. Assume $\alpha$ depends polynomially 
on $d^{-1}$.} 
Then there exists a constant $c^{\prime} > 0$ so that the
total regret achieved by our scheme is bounded as:

\begin{equation}
R_1+R_2+R_3 = \hemant{O\left(\frac{k^{13} d^2 \sigma^2 (\log n)^{4}}{\alpha^6} n^{\frac{4}{k+2}}
+ n^{\frac{1+k}{2+k}} (\log n)^{\frac{1}{2+k}} \right)} \label{eq:tot_regret_bd}
\end{equation}

\hemant{after $n$ rounds.}
\end{theorem} 
\begin{proof}
We first observe that the choice $f = \frac{1}{\sqrt{k}} \left(\frac{\log n}{n}\right)^{\frac{1}{k+2}}$ 
results in $n_1 = \frac{k^8 d^2 \sigma^2}{\alpha^2} (n/\log n)^{\frac{4}{k+2}} m_{\calX}^2 m_{\Phi} m$. 
While $n_1 = o(n)$ when $k \geq 3$, we also necessarily require $n_1 < n$. This is however ensured for $n$ satisfying 
\begin{equation} \label{eq:linpar_bandits_n_cond}
 n > \left(\frac{k^8 d^2 \sigma^2}{\alpha^2} m_{\calX}^2 m_{\Phi} m \right)^{\frac{k+2}{k-2}}. 
\end{equation}
For the stated choices of $m_{\calX}, m_{\Phi}, m$, \eqref{eq:linpar_bandits_n_cond} leads to a bound on $n$ that 
is clearly polynomial in $k,d$. Assuming $n$ satisfies \eqref{eq:linpar_bandits_n_cond}, the stated choice of $f$ also results in:
\begin{equation*}
 n_2 \sqrt{k} f = O(n^{\frac{1+k}{2+k}} (\log n)^{\frac{1}{2+k}}).
\end{equation*}

Upon using this in \eqref{eq:tot_regret_exp1} we obtain:
\begin{equation}
R_1+R_2+R_3 = O\left(\frac{k^8 d^2 \sigma^2}{\alpha^2} \left(\frac{n}{\log n}\right)^{\frac{4}{k+2}} m_{\calX}^2 m_{\Phi} m + 
n^{\frac{1+k}{2+k}} (\log n)^{\frac{1}{2+k}} + \hemant{n_2 k^{3/2} p} \right) \label{eq:tot_regret_exp2}
\end{equation}
In order to choose $m_{\calX}$ and $m_{\Phi}$ we simply note from \eqref{eq:success_prob_exp} that the choices
\begin{equation}
m_{\calX} = \frac{2k C_2^2}{\alpha\rho^2}\log(k/\hemant{\ptil}), \quad m_{\Phi} = \frac{4k(d+m_{\calX}+1)u(\delta)c_1}{q(\delta)} 
\end{equation}
for suitable constant $c_1 > 1$, \hemant{$\ptil = \frac{1}{k^{3/2} n}$, results in $k^{3/2} n p = O(1)$, under 
the assumption $\alpha$ depends polynomially on $d^{-1}$ (and hence $m = \max\set{d,m_{\calX}} = m_{\calX}$)}. 
Then plugging the above choice of $m_{\calX}$ and $m_{\Phi}$ in \eqref{eq:tot_regret_exp2}
we obtain:
\begin{align*}
R_1+R_2+R_3 &= O\left(\frac{k^9 d^2 \sigma^2}{\alpha^2} \left(\frac{n}{\log n}\right)^{\frac{4}{k+2}} \hemant{m_{\calX}^3 (d+m_{\calX})} + 
n^{\frac{1+k}{2+k}} (\log n)^{\frac{1}{2+k}}  \right) \\
&= O\left(\frac{k^9 d^2 \sigma^2}{\alpha^2} \left(\frac{n}{\log n}\right)^{\frac{4}{k+2}} \hemant{(k\alpha^{-1}\hemant{\log n})^4} + 
n^{\frac{1+k}{2+k}} (\log n)^{\frac{1}{2+k}}  \right) \\
&= \hemant{O\left(\frac{k^{13} d^2 (\hemant{\log n})^4 \sigma^2}{\alpha^6}  n^{\frac{4}{k+2}} + 
n^{\frac{1+k}{2+k}} (\log n)^{\frac{1}{2+k}}  \right).}
\end{align*}
\end{proof}

%
\begin{remark} \label{rem:reg_bd_k_12}
Upon examining the regret bound in Theorem \ref{thm:main_thm_reg_bds}, we observe that the 
dependency on $n$ is \hemant{$n^{\frac{1+k}{2+k}} (\log n)^{\frac{1}{2+k}}$ when $k > 3$}. For $k=1,2$ however, the term 
$n^{\frac{4}{k+2}}$ is super-linear in $n$ rendering the bound meaningless. 
This is handled by changing the choice of $f$ in Theorem \ref{thm:main_thm_reg_bds} to 
$f = \frac{1}{\sqrt{k}}(\log n/n)^{\frac{0.5}{k+2}}$. By following the steps in the 
proof, one can then verify that the regret is bounded by:
\hemant{\begin{equation*}
O\left(\frac{k^{13} d^2 \sigma^2 (\log n)^{4}}{\alpha^6}  n^{\frac{2}{k+2}}
+ n^{\frac{1.5+k}{2+k}} (\log n)^{\frac{0.5}{2+k}} \right). 
\end{equation*}}
We see that the dependency on $n$ is now \hemant{$n^{\frac{1.5+k}{2+k}} (\log n)^{\frac{0.5}{2+k}}$} for $k \geq 1$.
\end{remark}

Our complete scheme which we name CAB-LP($d,k$) (Continuum armed bandit of $k$ linear parameters in $d$ dimensions) 
is presented as Algorithm \ref{alg:cont_arm_band_lin_k}.

\begin{algorithm} 
\caption{Algorithm CAB-LP($d,k$)} \label{alg:cont_arm_band_lin_k}
\begin{algorithmic}
\State \textbf{Input:} $k,d,n,C_2,\sigma$.

\State Choose $0 < \delta < \sqrt{2}-1$; \hemant{$\rho \in (0,1)$, $\ptil = \frac{1}{n k^{3/2}}$} and $c_1 > 1$. 
       Choose $\alpha$ according to model assumption on mean reward function.

\State Set $f = \frac{1}{\sqrt{k}}\left(\frac{\log n}{n}\right)^{\frac{1}{k+2}}$, $m_{\calX} = \frac{2k C_2^2}{\alpha\rho^2}\log(k/\hemant{\ptil})$ and 
       $m_{\Phi} = \frac{4k(d+m_{\calX}+1)u(\delta)c_1}{q(\delta)}$.

\State Choose re-sampling factor $N$ according to Proposition \ref{prop:regret_bd_phase1}. 

\State Choose step size $\epsilon$ as in \eqref{eq:eps_cond} with $\sigma \leftarrow \sigma/\sqrt{N}$.

\State \textbf{PHASE 1 (Subspace recovery phase)}  $t = 1,\dots,N m_{\calX} (m_{\Phi}+1)$

\begin{itemize}[leftmargin=1.5cm]
\item Create random sampling sets $\calX$ and $\Phi$ as explained in Section \ref{subsec:analysis_subsp_recov} so that
       $\abs{\calX} = m_{\calX}$ and $\abs{\Phi} = m_{\calX}m_{\Phi}$.

\item For $t=1,\dots,m_{\calX}(m_{\Phi}+1)$ collect rewards $(r_j(\vecx_j))_{j=1}^{m_{\calX}}$ and 
       $(r_{m_{\calX} + ij}(\vecx_j + \epsilon\phi_{i,j}))_{j=1,i=1}^{m_{\calX},m_{\Phi}}$.

\item Re-sample and average the reward values $N$ times at each $\vecx$ and $\vecx + \epsilon \phi$ respectively ($\vecx \in \calX, \phi \in \Phi$). Form 
      measurement vector $\vecy$ as in \eqref{eq:meas_vec_form} with the averaged reward values.

\item Obtain $\mathatX_{DS}^{(k)}$ as best rank-$k$ approximation to solution of matrix DS \eqref{eq:mat_DS_form} and set $\mathatA^T$ to left 
      singular vector matrix of $\mathatX_{DS}^{(k)}$.
\end{itemize}

\State \textbf{PHASE 2 (Optimization phase)}  $t = N m_{\calX} (m_{\Phi} + 1) + 1,\dots,n$

\begin{itemize}[leftmargin=1.5cm]
\item Employ CAB1 algorithm \cite{Kleinberg04} on $\calP := \set{\widehat{\matA}^T\vecy \in \mathbb{R}^d : \vecy \in B_{k}(1+\nu)}$.

\end{itemize}

\end{algorithmic}
\end{algorithm}

\subsection{Remarks on the tractability parameter $\alpha$} \label{subsec:tract_param_alpha}
We now proceed to comment on the parameter $\alpha$ of our scheme which
also appears in our regret bounds. Recall from Section \ref{sec:prob_setup} that $\alpha$ measures the conditioning of the following matrix:
\begin{equation}
H^r := \int_{\mathbb{S}^{d-1}} \nabla\bar{r}(\vecx) \nabla\bar{r}(\vecx)^{T} d\vecx 
= \matA^T \cdot \int_{\mathbb{S}^{d-1}} \nabla\bar{g}(\matA\vecx) \nabla\bar{g}(\matA\vecx)^{T} d\vecx \cdot \matA.
\end{equation}
More specifically, we assume that the mean reward function $\bar{r}$ is such that:
\begin{equation}
\sigma_1(H^{r}) \geq \sigma_2(H^{r}) \geq \dots \geq \sigma_k(H^{r}) \geq \alpha > 0
\end{equation}
where $\sigma_i(H^r)$ denotes the $i^{th}$ singular value of $H^{r}$. In other words $\alpha$ measures how far away from $0$ the lowest 
singular value of $H^{r}$ is, implying that a larger $\alpha$ indicates a well conditioned $H^{r}$. A natural question that arises now is 
on the behaviour of $\alpha$ - in particular on its dependence on dimension $d$ and number of linear parameters $k$. To this end we first note that
the parameter typically decays with increase in $d$. In fact for $k > 1$ this would always be the case since as $d \rightarrow \infty$ the matrix 
$H^{r}$ would converge to a rank-$1$ matrix \cite{Tyagi2014_acha}.

We also note from our derived regret bounds that in case $\alpha \rightarrow 0$ exponentially fast as $d \rightarrow \infty$ then our regret bounds will 
have a factor exponential in $d$ which is clearly undesirable. Hence it is important to define classes of functions for which $\alpha$ 
provably decays polynomially as $d \rightarrow \infty$ so that our regret bounds depend \textit{at most polynomially} on dimension $d$. We now state 
the following result by Tyagi et al. \cite{Tyagi2014_acha} which defines such a class of functions for which $\alpha = \Theta(d^{-1})$. 

\begin{proposition}[\cite{Tyagi2014_acha}] \label{prop:tract_cond_alpha}
\hemant{Assume that $g:B_{k}(1) \rightarrow \mathbb{R}$, with $g$ being a $\mathcal{C}^2$ function, has Lipschitz continuous second order partial derivatives
in an open neighborhood of the origin, $\mathcal{U}_{\theta} = B_{k}(\theta)$ for some fixed $\theta$ (depending only on $k$ with $k$ fixed):
\begin{equation*}
\frac{\abs{\frac{\displaystyle \partial^2 g}{\displaystyle \partial y_i \partial y_j}(\vecy_1) - 
\frac{\displaystyle \partial^2 g}{\displaystyle \partial y_i \partial y_j}(\vecy_2)}}{\norm{\vecy_1 - \vecy_2}} 
< L_{i,j} \quad \forall \vecy_1,\vecy_2 \in \mathcal{U}_{\theta}, \vecy_1 \neq \vecy_2, \ i,j= 1,\dots,k. 
\end{equation*}
Denoting $L = \max_{1 \leq i,j \leq k} L_{i,j}$, assume that $\nabla^2 g(\veco)$ is full rank. 
Then provided that $\grad g(\mathbf{0}) = \mathbf{0}$, we have $\alpha = \Theta(1/d)$ as $d \rightarrow \infty$.}
\end{proposition}

\hemant{\begin{remark}
It is worth mentioning that the Proposition as stated in \cite{Tyagi2014_acha} has a couple of minor inaccuracies in the proof\footnote{This is fixed in the arxiv version of the paper.}. 
Firstly, the condition $\grad g(\mathbf{0}) = \mathbf{0}$ is not mentioned. 
This is probably not completely necessary and could be relaxed, but one would then require the parameter $L$ to be sufficiently small. 
Secondly, the result is stated for $\theta = O(d^{-(s+1)})$ for some $s > 0$. However it is important for $\theta$ to be fixed 
independent of $d$, as otherwise, the $1/d$ scaling would not hold. 
\end{remark}}

The class of functions defined in the above Proposition covers a number of function models such as sparse additive models of the 
form $\sum_{i=1}^{k} g_i(\vecy)$ where $g_i$'s are kernel functions \cite{Li05}.
Further details in this regard are provided by Tyagi et al. \cite[Section 5]{Tyagi2014_acha}.
Finally, in light of the above discussion on $\alpha$ we arrive at the following Corollary 
of Theorem \ref{thm:main_thm_reg_bds} with the help of Proposition \ref{prop:tract_cond_alpha}.
\begin{corollary}
Assume that the mean reward function $\bar{r}: B_{d}(1+\nu) \rightarrow \mathbb{R}$ where $\bar{r}(\vecx) = \bar{g}(\matA\vecx)$ is such
that $\bar{g}$ satisfies the conditions of Proposition \ref{prop:tract_cond_alpha}. 
Then, \hemant{for $k \geq 3$}, there exists a constant $c^{\prime} > 0$ so that the total regret achieved by Algorithm CAB-LP(d,k) is bounded as:
\begin{equation} \label{eq:main_thm_corr}
R_1+R_2+R_3 = \hemant{O\left(k^{13} d^8 \sigma^2 (\log n)^{4} n^{\frac{4}{k+2}}
+ n^{\frac{1+k}{2+k}} (\log n)^{\frac{1}{2+k}} \right)},
\end{equation}
\hemant{after $n$ rounds}. 
\end{corollary}

%% file: bandits_conclusion.tex
\section{Concluding Remarks} \label{sec:conclusion}
To summarize, we considered a stochastic continuum armed bandit problem where the reward functions 
reside in a high dimensional space of dimension $d$ but intrinsically depend on $k$-linear combinations
of the $d$ coordinate variables. Assuming the time horizon $n$
to be known we derived a randomized algorithm that achieves a cumulative regret bound of
$O(C(k,d) n^{\frac{1+k}{2+k}} (\log n)^{\frac{1}{k+2}})$ with high 
probability where $C(k,d)$ is at most polynomial in $k,d$. Our algorithm combines results
from low rank matrix recovery literature with existing results on continuum armed bandits.

We noted earlier that recently, Djolonga et al. \cite{Djolonga13} consider the same problem as in this paper with the difference that
the mean reward functions are assumed to reside in a RKHS (Reproducible Kernel Hilbert Space).
They consider the Bayesian optimization framework and present an algorithm which has the same idea as ours in the 
sense of first estimating the unknown subspace spanned by the linear parameters and then performing Bayesian optimization on the estimated
subspace. Furthermore their algorithm also achieves this by careful allocation of the sampling budget amongst the two phases.

\subsection*{Improved regret bounds} We now mention that the regret bounds derived in this paper can possibly be
sharpened by employing recent results from finite armed bandit literature. For instance, if the range of the reward functions was restricted to be $[0,1]$ 
then one can simply use the INF algorithm \cite{Audibert10} as a sub-routine in the CAB1 algorithm \cite{Kleinberg04}
to get rid of the \hemant{$(\log n)^{\frac{1}{k+2}}$ factor appearing in \eqref{eq:main_thm_corr},\eqref{eq:tot_regret_bd}}. 
When the range of the reward functions is $\matR$, as is the case in our setting, it seems possible 
to consider a variant of the MOSS algorithm \cite{Audibert10} along with proof techniques considered in a modified UCB-1 
algorithm in Section 2 of \cite{KleinbergPhd} to remove the \hemant{$(\log n)^{\frac{1}{k+2}}$} factor from the regret bound. 

\subsection*{Future work} For future work it would be interesting to consider the setting where the time horizon $n$ is
unknown to the algorithm and to prove regret bounds for the same. In particular, it would be interesting to derive algorithms
which do not involve recovering an approximation of the unknown $k$ dimensional subspace spanned by the $k$ linear parameters. 
Lastly we mention other directions such as an adversarial version of our problem where the reward functions are
chosen arbitrarily by an adversary and also a setting where the unknown matrix $\matA$ is allowed to change across time.

%% file: bandits_appendix.tex
%
\appendix
\section{Proofs} \label{sec:appendix_proofs}
%
\subsection{Proof of Lemma \ref{lem:init_bound_R3}}
\begin{proof}
\hemant{We first observe that for any given $\mathatA$:
\begin{align}
\bar{r}(\vecx^{*}) - \bar{r}(\vecx^{**}) &= [\bar{g}(\matA\vecx^{*}) - \bar{g}(\matA\vecx^{**})] \\
&= [\bar{g}(\matA\vecx^{*}) - \bar{g}(\matA\widehat{\matA}^T\widehat{\matA}\vecx^{**})] \label{eq:third} \\
&\leq [\bar{g}(\matA\vecx^{*}) - \bar{g}(\matA\widehat{\matA}^T\widehat{\matA}\vecx^{*})] \label{eq:fourth} \\
&\leq C_2 \sqrt{k} \norm{\matA\vecx^{*} - \matA\widehat{\matA}^T\widehat{\matA}\vecx^{*}} \label{eq:fifth} \\
&\leq C_2 \sqrt{k} (1+\nu) \norm{\matA - \matA\widehat{\matA}^T\widehat{\matA}}_F \label{eq:six}\\
&= \frac{C_2 \sqrt{k} (1+\nu)}{\sqrt{2}} \norm{\matA^T\matA - \widehat{\matA}^T\widehat{\matA}}_F. \label{eq:seven}
\end{align}}
In \eqref{eq:third} we used the fact that $\vecx^{**} = \widehat{\matA}^T\widehat{\matA}\vecx^{**}$ since $\vecx^{**} \in \calP$. In \eqref{eq:fourth}
we used the fact that $\bar{g}(\matA\widehat{\matA}^T\widehat{\matA}\vecx^{**}) \geq \bar{g}(\matA\widehat{\matA}^T\widehat{\matA}\vecx^{*})$
since $\widehat{\matA}^T\widehat{\matA}\vecx^{*} \in \calP$ and $\vecx^{**} \in \calP$ is an optimal strategy.
\eqref{eq:fifth} follows from the mean value theorem along with the smoothness assumption made in \eqref{eq:prob_setup_smooth_assump}.
In \eqref{eq:six} we used the simple inequality : $\norm{\matB\vecx} \leq \norm{\matB}_F \norm{\vecx}$.
Obtaining \eqref{eq:seven} from \eqref{eq:six} is a straightforward exercise. \hemant{Lastly, the stated bound on $R_3$ follows easily  
via the law of total expectation, and by noting that the bound $\norm{\matA^T\matA - \widehat{\matA}^T\widehat{\matA}}_F = O(k)$ always holds.}
\end{proof}
%
\subsection{Proof of Lemma \ref{lem:recov_res_DS}}
\begin{proof}
We first have the following result by simply using Theorem 1 in \cite{Candes2010} in our setting for bounding the error of the matrix Dantzig selector.
\begin{theorem}
For any $\matX \in \mathbb{R}^{d \times m_{\calX}}$ such that rank($\matX$) $\leq k$ let $\mathatX_{DS}$ be the
solution of \eqref{eq:mat_DS_form}. If $\delta_{4k} < \delta < \sqrt{2}-1$ and
$\norm{\Phi^{*}(\vecH + \vecN)} \leq \lambda$ then we have with probability at least $1-2 e^{-m_{\Phi}q(\delta) + 4k(d+m_{\calX}+1)u(\delta)}$ that
\begin{equation*}
 \norm{\matX - \mathatX_{DS}}_F^2 \leq C_0 k \lambda^2 
\end{equation*}
where $C_0$ depends only on the isometry constant $\delta_{4k}$.
\end{theorem}
What remains to be found for our purposes is $\lambda$ which is a bound on $\norm{\Phi^{*}(\vecH + \vecN)}$. Firstly note that 
$\norm{\Phi^{*}(\vecH + \vecN)} \leq \norm{\Phi^{*}(\vecH)} + \norm{\Phi^{*}(\vecN)}$. From Lemma 1 and Corollary 1 of \cite{Tyagi2014_acha}
we have that:
\begin{equation*}
 \norm{\Phi^{*}(\vecH)} \leq \frac{C_2 \epsilon d m_{\calX} k^2}{2\sqrt{m_{\Phi}}}(1+\delta)^{1/2}
\end{equation*}
holds with probability at least $1-2 e^{-m_{\Phi}q(\delta) + 4k(d+m_{\calX}+1)u(\delta)}$ where $\delta$ is such that $\delta_{4k} < \delta < \sqrt{2}-1$.
Next we note that $\vecN = [N_1 N_2 \dots N_{m_{\Phi}}]$ where
\begin{equation*}
 N_i = \underbrace{\frac{1}{\epsilon}\sum_{j=1}^{m_{\calX}} \eta_j}_{L_{1,i}} - \underbrace{\frac{1}{\epsilon}\sum_{j=1}^{m_{\calX}}\eta_{i,j}}_{L_{2,i}}
\end{equation*}
with $\mathbf{L_1} = [L_{1,1} \dots L_{1,m_{\Phi}}]$ and $\mathbf{L_2} = [L_{2,1} \dots L_{2,m_{\Phi}}]$ so that $\vecN = \mathbf{L_1} - \mathbf{L_2}$. 
We then have that $\norm{\Phi^{*}(\vecN)} \leq \norm{\Phi^{*}(\mathbf{L_1})} + \norm{\Phi^{*}(\mathbf{L_2})}$. 
By using Lemma 1.1 of \cite{Candes2010} and denoting $m=\max\set{d,m_{\calX}}$ we first have that:
\begin{equation}
 \norm{\Phi^{*}(\mathbf{L_1})} \leq \frac{2\gamma\sigma}{\epsilon} \sqrt{(1+\delta) m_{\Phi} m_{\calX} m}  \label{eq:bound_L1_noise}
\end{equation}
holds with probability at least $1-2 e^{-c m}$ where $c = \frac{\gamma^2}{2} - 2\log 12$ and $\gamma > 2\sqrt{\log 12}$.
This can be verified using the proof technique of Lemma 1.1 of \cite{Candes2010} by taking care of the fact that the entries
of $\mathbf{L}_1$ are correlated as they are identical copies of the same Gaussian random variable $\frac{1}{\epsilon}\sum_{j=1}^{m_{\calX}} \eta_j$.
Furthermore we also have that:
\begin{equation}
 \norm{\Phi^{*}(\mathbf{L_2})} \leq \frac{2\gamma\sigma}{\epsilon} \sqrt{(1+\delta) m_{\calX} m}  \label{eq:bound_L2_noise}
\end{equation}
holds with probability at least $1-2 e^{-c m}$ with constants $c,\gamma$ as defined earlier. This is verifiable easily using the proof technique
Lemma 1.1 of \cite{Candes2010} as the entries of $\mathbf{L}_2$ are i.i.d Gaussian random variables. Combining \eqref{eq:bound_L1_noise} 
and \eqref{eq:bound_L2_noise} we then have that the following holds true with probability at least $1-4 e^{-c m}$.
\begin{equation}
  \norm{\Phi^{*}(\mathbf{L_1})} + \norm{\Phi^{*}(\mathbf{L_2})} \leq \frac{4\gamma\sigma}{\epsilon} 
	\sqrt{(1+\delta) m_{\calX} m_{\Phi} m}.
\end{equation}
Lastly, it is fairly easy to see that $\norm{\mathatX_{DS}^{(k)} - \matX}_F \leq 2 \norm{\mathatX_{DS} - \matX}_F$ where $\mathatX_{DS}^{(k)}$
is the best rank $k$ approximation to $\mathatX_{DS}$ (see for example the proof of Corollary 1 in \cite{Tyagi2014_acha}).
Combining the above observations we arrive at the stated error bound with probability at least
$1-2 e^{-m_{\Phi}q(\delta) + 4k(d+m_{\calX}+1)u(\delta)} - 4 e^{-c m}$.
\end{proof}
%
\subsection{Proof of Lemma \ref{lem:recov_res_subspace}}
\begin{proof}
Let $\tau$ denote the bound on $\norm{\mathatX_{DS}^{(k)} - \matX}_F$ as stated in Lemma \ref{lem:recov_res_DS}.
We make use of Lemma 2 of \cite{Tyagi2014_acha} which gives us that if $\tau < \frac{\sqrt{(1-\rho)m_{\calX}\alpha k}}{\sqrt{k}+\sqrt{2}}$ holds then
it implies that
\begin{equation}
 \norm{\mathatA^T\mathatA - \matA^T\matA}_F \leq \frac{2\tau}{\sqrt{(1-\rho)m_{\calX}\alpha} - \tau} \label{eq:perturb_bound_subsp}
\end{equation}
holds true for any $0 < \rho < 1$ with probability at least
\begin{equation*}
1 - 2\exp(-m_{\Phi}q(\delta) + 4k(d+m_{\calX}+1)u(\delta)) - 4\exp(-cm) -k\exp\left(-\frac{m_{\calX}\alpha\rho^2}{2k C_2^2}\right).
\end{equation*}
The proof makes use of Weyl's inequality \cite{Weyl1912} and Wedin's perturbation bound \cite{Wedin1972}. Therefore upon using the value of $\tau$
we have that $\tau < f \frac{\sqrt{(1-\rho)m_{\calX}\alpha k}}{\sqrt{k}+\sqrt{2}}$ holds for any $0 < f < 1$ if:
\begin{align}
C_0^{1/2} k^{1/2} (1+\delta)^{1/2} \left(\frac{C_2\epsilon d m_{\calX} k^2}{\sqrt{m_{\Phi}}} + \frac{8\gamma\sigma\sqrt{m_{\calX} m_{\Phi} m}}{\epsilon}\right) 
&< f \frac{\sqrt{(1-\rho)m_{\calX}\alpha k}}{\sqrt{k}+\sqrt{2}} \\
\Leftrightarrow \overbrace{C_2 d k^2}^{a_1} \epsilon \sqrt{\frac{m_{\calX}}{m_{\Phi}}} + \frac{8\gamma\sigma\sqrt{m_{\Phi} m}}{\epsilon}
&< f \left(\overbrace{\frac{1}{C_0^{1/2} (1+\delta)^{1/2}} \frac{\sqrt{(1-\rho)\alpha}}{\sqrt{k}+\sqrt{2}}}^{b_1}\right) \\
\Leftrightarrow a_1 \sqrt{\frac{m_{\calX}}{m_{\Phi}}} \epsilon^2 - f b_1 \epsilon + 8\gamma\sigma\sqrt{m_{\Phi} m} &< 0. \label{eq:eps_ineq}
\end{align}
From \eqref{eq:eps_ineq} we get the stated condition on $\epsilon$. Lastly upon using $\tau < \frac{f\sqrt{(1-\rho)m_{\calX}\alpha k}}{\sqrt{k}+\sqrt{2}}$ 
in \eqref{eq:perturb_bound_subsp} we obtain the stated bound on $\norm{\mathatA^T\mathatA - \matA^T\matA}_F$.
\end{proof}